\documentclass[11pt]{article}

\usepackage[final]{acl}

\usepackage{times}
\usepackage{latexsym}
\renewcommand{\cite}{\citep}
\usepackage[T1]{fontenc}    
\usepackage{hyperref}       
\usepackage{url}            
\usepackage{booktabs}       
\usepackage{amsfonts}       
\usepackage{nicefrac}       
\usepackage{microtype}      
\usepackage{xcolor}         
\definecolor{citecolor}{HTML}{0071BC}
\definecolor{linkcolor}{HTML}{D32F2F}
\definecolor{cellcolor}{HTML}{E3F2FD}
\definecolor{red}{HTML}{D32F2F}
\definecolor{magenta}{HTML}{D81B60}

\usepackage{svg}
\svgsetup{inkscapearea=page}

\usepackage{amsmath}
\usepackage{amssymb}
\usepackage{mathtools}
\usepackage{amsthm}
\usepackage{amsmath, amssymb, amsthm}
\usepackage{amsfonts}
\usepackage[capitalize,noabbrev]{cleveref}

\theoremstyle{plain}
\newtheorem{theorem}{Theorem}[section]

\theoremstyle{definition}

\theoremstyle{remark}

\usepackage[textsize=tiny]{todonotes}
\usepackage{microtype}
\usepackage{graphicx}
\usepackage{booktabs} 

\usepackage{pgfplots}
\pgfplotsset{compat = newest}
\usepackage{multirow}
\usepackage{enumitem}
\usepackage{caption}
\usepackage{fancybox}
\usepackage{framed}
\usepackage{tcolorbox}

\usepackage{algorithm}
\usepackage{mathrsfs}
\usepackage{mathtools}
\usepackage{algorithmic}
\newcommand{\rcomment}[1]{\hfill{\textit{#1}}}
\usepackage{tcolorbox}
\usepackage{listings}
\usepackage{makecell}
\usepackage{colortbl}
\usepackage{color}
\usepackage{wrapfig}
\usepackage{cancel}
\usepackage{soul,xcolor}
\usepackage{pifont}
\usepackage{tikz}
\usetikzlibrary{arrows.meta, positioning, fit, backgrounds}
\usetikzlibrary{shapes, positioning, arrows.meta, calc, decorations.pathmorphing, quotes}
\usepackage{todonotes}
\tcbuselibrary{breakable}
\usepackage{hyperref}
\usepackage{makecell}
\usepackage{url}
\hypersetup{colorlinks=true, linkcolor=linkcolor, citecolor=citecolor,urlcolor=magenta}
\usepackage{booktabs}
\usepackage{hyperref}
\usepackage{url}
\usepackage{graphicx}
\usepackage{subcaption}
\usepackage{amsthm}

\usepackage{enumitem}
\usepackage{enumitem}
\setlist[itemize]{noitemsep, topsep=2pt, parsep=2pt, partopsep=0pt}
\setlist[itemize]{leftmargin=2em}
\setlist[enumerate]{leftmargin=2em}

\DeclareMathOperator{\mean}{mean}
\DeclareMathOperator{\std}{std}

\usepackage[T1]{fontenc}

\usepackage[utf8]{inputenc}

\usepackage{microtype}

\usepackage{inconsolata}

\usepackage{graphicx}

\title{Targeted Exploration via Unified Entropy Control for Reinforcement Learning}

\author{
 \textbf{Chen Wang\textsuperscript{1,2}},
 \textbf{Lai Wei\textsuperscript{2,3}},
 \textbf{Yanzhi Zhang\textsuperscript{2,4}},
 \textbf{Chenyang Shao\textsuperscript{2,5}},
\\
 \textbf{Zedong Dan\textsuperscript{2,6}},
 \textbf{Weiran Huang\textsuperscript{3}},
 \textbf{Ge Lan\textsuperscript{1,*}},
 \textbf{Yue Wang \textsuperscript{2,*}},
\\
\\
 \textsuperscript{1}College of Software, Nankai University
 \textsuperscript{2}Zhongguancun Academy
 \textsuperscript{3}Shanghai Jiao Tong University\\
 \textsuperscript{4}Chinese Academy of Sciences
 \textsuperscript{5}Tsinghua University
  \textsuperscript{6}Sun Yat-sen Univeristy
\\
 \small{
   \textbf{\textsuperscript{*}Correspondence:} 
   \href{mailto:email@domain}{lange@nankai.edu.cn, yuewang@bza.edu.cn}
 }
}

\begin{document}
\maketitle
\begin{abstract}
Recent advances in reinforcement learning (RL) have improved the reasoning capabilities of large language models (LLMs) and vision-language models (VLMs). However, the widely used Group Relative Policy Optimization (GRPO) consistently suffers from entropy collapse, causing the policy to converge prematurely and lose diversity. Existing exploration methods introduce additional bias or variance during exploration, making it difficult to maintain optimization stability. We propose Unified Entropy Control for Reinforcement Learning (UEC-RL), a framework that provides targeted mechanisms for exploration and stabilization. UEC-RL activates more exploration on difficult prompts to search for potential and valuable reasoning trajectories. In parallel, a stabilizer prevents entropy from growing uncontrollably, thereby keeping training stable as the model consolidates reliable behaviors. Together, these components expand the search space when needed while maintaining robust optimization throughout training. Experiments on both LLM and VLM reasoning tasks show consistent gains over RL baselines on both Pass@1 and Pass@$k$. On Geometry3K, UEC-RL achieves a 37.9\% relative improvement over GRPO, indicating that it sustains effective exploration without compromising convergence and underscoring UEC-RL as a key for scaling RL-based reasoning in large models. Our code is available at \url{https://github.com/597358816/UEC-RL}.
\end{abstract}

\begin{figure*}
    \centering
    \includegraphics[width=\linewidth]{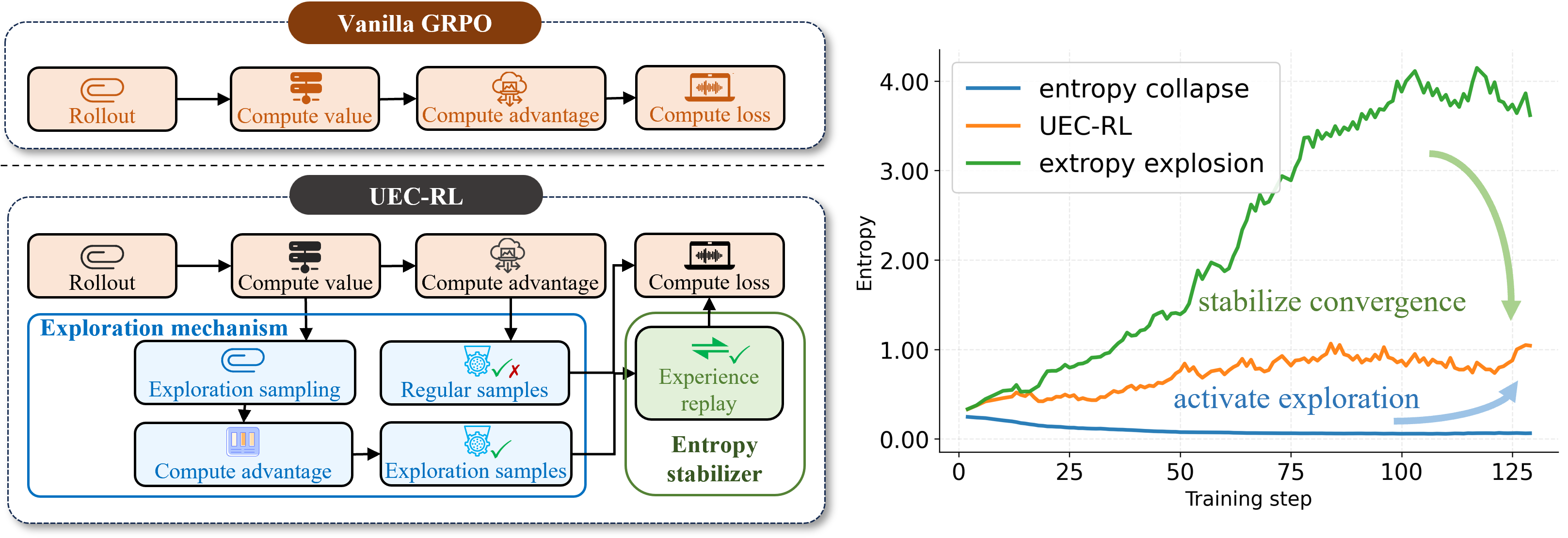}
    \caption{\small Illustration of UEC-RL. UEC-RL balances exploration and stabilization, keeping entropy within an optimal operating range.}
    \label{fig:uecrl}
    \vspace{-8pt}
\end{figure*}

\section{Introduction}

\begingroup
\renewcommand\thefootnote{}
\footnotetext{Accepted for publication in Findings of the 64th Annual Meeting of the Association for Computational Linguistics (ACL 2026).}
\addtocounter{footnote}{-1}
\endgroup

\paragraph{Reinforcement learning (RL)} has become a central paradigm in the post-training of large language models (LLMs) and vision-language models (VLMs) \cite{glm2024chat, touvron2023llama}. Early RLHF methods, such as Proximal Policy Optimization (PPO) and Direct Preference Optimization (DPO), align model outputs with human preferences using preference-based rewards \cite{schulman2017proximal, rafailov2023direct, zhong2024dpo, wang2024comprehensive}. More recently, Reinforcement Learning with Verifiable Rewards (RLVR) has emerged as a scalable alternative by leveraging automatically verifiable supervision \cite{mroueh2025reinforcement}. Within this framework, Group Relative Policy Optimization (GRPO) was introduced as a lightweight PPO variant that removes the critic and estimates advantages via group-normalized rewards, achieving strong efficiency and competitive reasoning performance \cite{shao2024deepseekmath, liu2024deepseek, guo2025deepseek}.

\paragraph{Exploration} in RL is the process of guiding the policy to observe sufficiently diverse and informative samples during training, so that optimization can operate over a broader solution space instead of collapsing to suboptimal behaviors early~\cite{sutton1998reinforcement, auer2002finite, strehl2008analysis, kolter2009near}. \textbf{Entropy} quantifies the policy’s uncertainty during inference and is commonly used as a proxy for exploration~ \cite{schulman2017equivalence, haarnoja2018soft, nachum2017bridging}, yet recent studies show that GRPO suffers from entropy collapse, where policy entropy rapidly drops and responses become highly convergent, severely limiting the discovery of potential and valuable trajectories  \cite{yue2025does, yu2025dapo}. DAPO slows entropy decay via a clip-higher strategy, but often at the cost of increased update variance and unstable training \cite{yu2025dapo}. Other methods introduce entropy bonuses or modified advantages \cite{cui2025entropy, cheng2025reasoning}, but the resulting exploration is frequently biased, as these approaches explicitly optimize entropy-related terms rather than task rewards. \textbf{Moreover, effective exploration should emphasize informative and high-quality diversity, rather than indiscriminately encouraging randomness}. This requires mechanisms that can suppress excessively high-entropy behaviors once exploration becomes unproductive, a capability that is largely missing in existing approaches. Overall, the field still lacks a mechanism for controlling entropy in both directions, exploration and stabilization.

To address this gap, we introduce UEC-RL, a framework that integrates exploration and stabilization within a single, coherent mechanism. Rather than merely slowing entropy collapse, UEC-RL actively adjusts the degree of exploration according to problem difficulty, enabling entropy to increase when deeper reasoning is required and allowing the model to access low-probability but informative trajectories that standard sampling often misses. At the same time, UEC-RL incorporates a stabilizing mechanism that restrains uncontrolled entropy growth, reinforces reliable behaviors, and guides the policy toward stable convergence as learning progresses. Through this coordinated design, UEC-RL dynamically balances exploration and exploitation throughout training.

Experiments across a wide range of LLM and VLM reasoning benchmarks show that UEC-RL delivers consistent gains over RL baselines. On the challenging Geometry3K dataset, it achieves a 37.9\% relative improvement in accuracy while maintaining more appropriate training dynamics. UEC-RL also provides robust improvements on Pass@$k$, demonstrating that UEC-RL is an effective method for advancing RL-based reasoning in large-scale models. Our contributions are summarized as follows:

\begin{itemize}[leftmargin=1em]
    \item We introduce UEC-RL, which provides bidirectional entropy regulation, enabling both controllable entropy increase for deep exploration and controllable entropy stabilization for reliable training. As shown in Fig.~\ref{fig:uecrl}, UEC-RL includes:
    \begin{itemize}[leftmargin=1.5em]
        \item A targeted exploration mechanism that activates high-entropy reasoning specifically on difficult problems, allowing the model to uncover low-probability but informative trajectories that standard sampling rarely reaches.
        \item A controllable entropy stabilizer that amplifies reliable gradients, suppresses unbounded exploration, and guides the policy toward stable convergence.
    \end{itemize}
    \item We demonstrate consistent improvements across LLM and VLM reasoning benchmarks, including strong gains on Geometry3K and Pass@$k$ evaluations, confirming the effectiveness and generality of the proposed entropy-control paradigm.
\end{itemize}

\begin{figure*}[t]
	\setlength{\belowdisplayskip}{-10pt} 
\begin{equation}
    \label{eq:GRPO}
 \begin{split}
    \small
  \mathcal{J}_{GRPO}(\theta) = & \mathbb{E}_{(q,a)\sim \mathcal{D},\ O=\{o_i\}_{i=1}^{G} \sim \pi_{\theta_{old}(\cdot | q)}} \\ &
   \frac{1}{G}\sum_{i=1}^{G}\frac{1}{|o_i|}\sum_{t=1}^{|o_i|}  \left\{\min \left[r_{i,t}(\theta)\hat{A}_{i,t},\ \text{clip} (r_{i,t}(\theta),1-\epsilon,1+\epsilon)\hat{A}_{i,t}\right] - \beta D_{KL}[\pi_{\theta}||\pi_{ref}]\right\},
 \end{split}
\end{equation}
\noindent \text{where 
$r_{i,t}(\theta) = \frac{\pi_{\theta}(o_{i,t}|q,o_i<t)}{\pi_{\theta_{old}}(o_{i,t}|q,o_i<t)}$, and $\hat{A}_{i,t} = \frac{R_i-\mean(\{R_i\}_{i=1}^{G})}{\std(\{R_i\}_{i=1}^{G})}.$}\\

 \centering
  \includegraphics[width=0.3\linewidth]{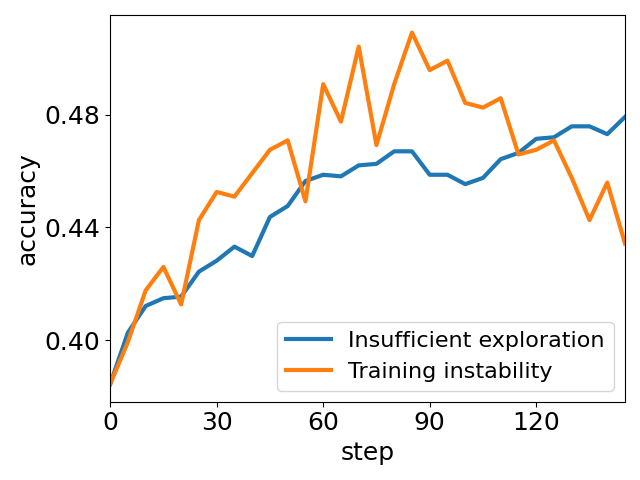}
  \includegraphics[width=0.3\linewidth]{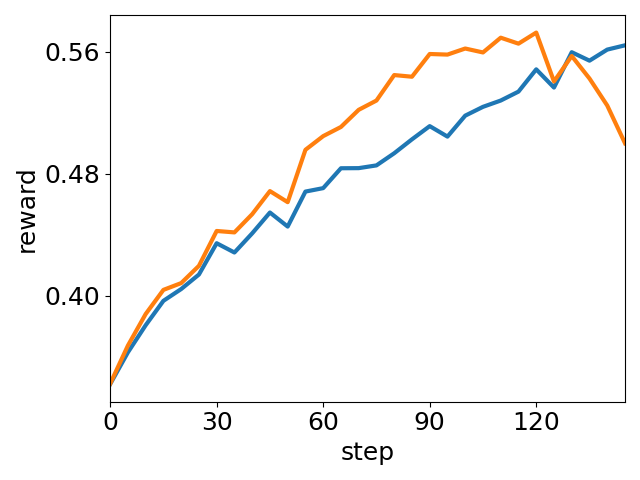}
  \includegraphics[width=0.3\linewidth]{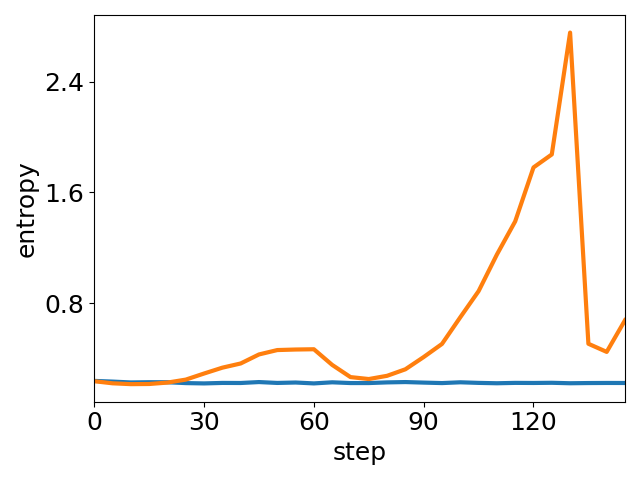}
 \caption{Two prominent optimization issues of GRPO on Geometry3K: insufficient exploration (entropy collapse) and unstable training dynamics.}
 \label{fig:GRPO-drawback}
\end{figure*}

\section{Related Work}

Recent post-training progress has shown that reinforcement learning is effective for improving reasoning in foundation models, offering a scalable way to refine long-chain decision behaviors beyond supervised fine-tuning \cite{openai2024reasoning}. Recent advances in aligning large language models (LLMs) and vision-language models (VLMs) have been driven by reinforcement learning techniques \cite{openai2023gpt4, team2024gemini1_5, zhu2023minigpt, wei2023instructiongpt, liu2023visual, team2025kimi1_5, yang2025qwen3}, most notably RLHF and policy optimization methods such as DPO, PPO, and GRPO \cite{rafailov2024direct, ouyang2022training, schulman2017proximal, shao2024deepseekmath}. More recently, Reinforcement Learning with Verifiable Rewards (RLVR) has shown strong performance in reasoning-intensive domains by removing learned reward models and relying on structured, verifiable supervision \cite{lambert2024tulu, guo2025deepseek}. In this setting, optimization efficiency and training dynamics become critical, particularly for GRPO-based methods that normalize rewards within rollout groups.

Exploration is a core component of reinforcement learning, enabling policies to escape local optima and discover high-reward behaviors \cite{sutton1998reinforcement}. In policy gradient methods, entropy is commonly used as a proxy for exploration to encourage behavioral diversity \cite{schulman2017proximal}. However, recent studies show that GRPO-based training often suffers from entropy collapse, resulting in highly convergent samples and insufficient exploration \cite{yue2025does, yu2025dapo}.  Existing remedies can be broadly grouped into two categories. The first modifies policy updates via clipping strategies \cite{yu2025dapo, hao2025rethinking, su2025gppo}. By relaxing the upper clipping bound, these methods can slow entropy contraction, but often amplify update variance and destabilize training. The second category promotes exploration through entropy bonuses, such as entropy regularization or entropy-shaped advantages \cite{adamczyk2025average, li2025entropy, hou2025advancing, tan2025gtpo, cui2025entropy, cheng2025reasoning}. While effective at increasing diversity, these approaches rely on coarse regularization and may introduce optimization bias, since exploration is driven by entropy-related objectives rather than task rewards.

\section{Preliminaries}

\subsection{RL baseline: GRPO}

GRPO has been widely used to improve the reasoning capabilities of large language models, particularly in mathematical problem-solving. Unlike PPO, GRPO removes the critic and estimates advantages using group-normalized rewards, resulting in significantly improved computational efficiency. Given a question $q$ with a verifiable answer $a$, GRPO samples $G$ responses $O=\{o_i\}_{i=1}^{G}$ from the old policy $\pi_{\theta_{\text{old}}}$ and updates the policy by maximizing the group-relative objective shown in Eq.~\eqref{eq:GRPO}. By normalizing rewards within each sampled group, GRPO provides a lightweight and scalable alternative to PPO and has demonstrated strong performance across diverse reasoning benchmarks.

\subsection{Limitations of GRPO}

Despite its empirical success, GRPO presents notable shortcomings when applied to complex reasoning tasks. Fig.~\ref{fig:GRPO-drawback} illustrates two representative failure modes on the Geometry3K dataset, which were observed under identical settings, revealing GRPO’s difficulty in maintaining adequate exploration and stable optimization dynamics.

\paragraph{Entropy collapse.}
As illustrated in Fig.~\ref{fig:GRPO-drawback}, policy entropy often decreases rapidly during training, causing the model to converge prematurely to low-diversity behaviors. This collapse is especially common in text-only reasoning, where training is relatively stable and the absence of entropy-increasing mechanisms prevents the model from exploring low-probability but informative trajectories. Consequently, the model activates only shallow reasoning patterns while deeper knowledge remains underutilized.

\paragraph{Training instability.}
A second failure mode commonly appears in more complex settings such as multimodal reasoning, as illustrated in Fig.~\ref{fig:GRPO-drawback}, where optimization exhibits sharp fluctuations. When sampled outputs vary greatly in correctness or reasoning difficulty, the group-normalized reward used by GRPO no longer provides sufficient variance reduction. As a result, gradient updates become brittle, causing unstable learning dynamics and occasionally leading to sudden performance degradation or even collapse.

These two issues share a structural root cause: GRPO lacks a mechanism for regulating entropy in either direction. It cannot actively increase entropy to enhance exploration on difficult prompts, nor can it stabilize entropy in high-variance regimes to ensure reliable convergence. This limitation motivates the development of a unified framework capable of dynamic, bidirectional entropy control.

\section{Methodology}

To address the entropy collapse and instability issues observed in GRPO, we introduce \textbf{UEC-RL}, which enables both controlled entropy increase for exploration and entropy reduction for stable convergence. This section first presents the targeted exploration mechanism, which adaptively activates high-entropy reasoning on difficult prompts, followed by the entropy-reducing replay mechanism that stabilizes learning and improves sample efficiency. The overall UEC-RL training procedure is summarized in Algorithm~\ref{alg:uec_rl}.

\subsection{Targeted Exploration Mechanism}
UEC-RL enhances GRPO’s limited exploratory capacity through a unified mechanism that (1) identifies prompts requiring deeper search, (2) expands the sampling space only when necessary, and (3) selectively retains informative trajectories. 

\paragraph{Expanding the exploration space.} If none of the initial $G$ rollouts solve a prompt, the prompt $\bar{\mathcal{D}}$ is marked as difficult, indicating insufficient exploration under the current policy. 
\vspace{-4pt}
\begin{equation*}
\begin{aligned}
\bar{\mathcal{D}}
=\Bigl\{(q,a)\;:\;& o_i \sim \pi_{\theta_{\text{old}}},\ i=1,\dots,G,\\
&\#\{\,i:\ R(a,o_i)>0\,\}=0
\Bigr\}.
\end{aligned}
\end{equation*}
\vspace{-8pt}

For such prompts, UEC-RL temporarily samples from a softened distribution with temperature $t'$, increasing the chance of uncovering low-probability but informative reasoning paths while leaving easy prompts unaffected.

\paragraph{Exploring informative trajectories.} From all collected samples, UEC-RL retains only two types of trajectories: regular samples $O_R$ with nonzero advantage and valuable samples  $O_H$ obtained under expanded sampling:
\begin{equation*}
\begin{aligned}
 O_R&=\left\{ \{o_i\}_{1:G} \sim \pi_{\theta_{old}}: (q,a)\sim \mathcal{D},\ \hat{A}_{i,t} \ne 0 \right\};  \\
O_H&=\left\{ \{o_i\}_{1:G'} \sim \pi^{t'}_{\theta_{old}}: (q,a)\sim \bar{\mathcal{D}}, \hat{A}_{i,t}\ge0 \right\}.  
\end{aligned}
\end{equation*}

\begin{algorithm}[t]
\caption{UEC-RL}
\label{alg:uec_rl}
\begin{algorithmic}[1]
\STATE \textbf{Input:} Dataset $\mathcal{D}$, policy $\pi_\theta$.
\STATE \textbf{Hypers:}
\STATE \hspace{1em}-- $G$: rollout group size.
\STATE \hspace{1em}-- $G'$: exploration group size, $G'>G$.
\STATE \hspace{1em}-- $t'$: exploration temperature, $t'>1$.
\STATE \hspace{1em}-- $s'$: replay size, a multiple of batch size.
\STATE \hspace{1em}-- $f_{\mathrm{replay}}$: replay frequency, positive integer.
\STATE Init queue $\mathcal{B}_{\mathrm{replay}}$ by size $s'$;
\REPEAT
  \STATE \rcomment{// Step 1: Regular rollout}
  \STATE Sample minibatch $\mathcal{B}_{\mathrm{data}}\subset\mathcal{D}$;
  \FOR{each $(q,a)\in\mathcal{B}_{\mathrm{data}}$}
    \STATE Sample $O \leftarrow \{o_i\}_{1:G}\sim \pi_{\theta_{\mathrm{old}}}$;
    \STATE Compute $R_i,\hat{A}_{i,t}$ of $O$;
    \IF{$\max_i R_i>0$}
      \STATE $O_R \leftarrow \{o_i\in O:\hat{A}_{i,t}\neq 0\}$;
    \ELSE
      \STATE \rcomment{// Step 2: Exploration rollout}
      \STATE Sample $O' \leftarrow \{o_i\}_{1:G'}\sim \pi^{t'}_{\theta_{\mathrm{old}}}$;
      \STATE Compute $R_i,\hat{A}_{i,t}$ of $O'$;
      \STATE $O_H \leftarrow \{o_i \in O':\hat{A}_{i,t}>0\}$;
      \STATE $\mathcal{O}_S \leftarrow O_H \cup \{o_i \in O:\hat{A}_{i,t}>0\}$;
      \STATE Push $\mathcal{O}_S$ into $\mathcal{B}_{\mathrm{replay}}$;
    \ENDIF
  \ENDFOR

  \STATE $\mathcal{O}_{\mathrm{eff}} \leftarrow \bigcup O_R \cup \bigcup O_H$;
  \STATE Compute $\pi_{\theta_{\mathrm{old}}}$ of $\mathcal{O}_{\mathrm{eff}}$;
  \STATE Update actor using $\mathcal{O}_{\mathrm{eff}}$;

  \STATE \rcomment{// Step 3: Replay stabilization}
  \IF{$\text{global\_step}\bmod f_{\mathrm{replay}}=0$}
    \STATE Sample $\mathcal{O}_S\subset\mathcal{B}_{\mathrm{replay}}$ and update actor;
  \ENDIF
\UNTIL{convergence}
\end{algorithmic}
\end{algorithm}

Low-advantage exploratory samples are filtered out because they tend to introduce noisy gradients that hinder optimization and generalization. These retained samples constitute the effective optimization set and enable exploration to be increased in a targeted manner.

By integrating difficulty detection, adaptive search expansion, and selective retention into a coherent procedure, UEC-RL activates exploratory behavior precisely where deeper reasoning is required. The mechanism is theoretically supported by the following result:

\begin{theorem}[Entropy Change]
\label{theo:entropy}
For a softmax policy updated by natural policy gradient with step size $\eta$,
\[
\begin{split}
   & H(\pi_\theta^{k+1}) - H(\pi_\theta^{k})
\approx~ \\ &
\quad -\eta\,\mathbb{E}_{s\sim d_\mu^{k}}
\mathrm{Cov}_{a\sim\pi_\theta^{k}(\cdot|s)} 
\bigl[\log\pi_\theta^{k}(a|s),~A^{\pi^k}(s,a)\bigr].
\end{split}
\]
\end{theorem}
This theorem was first introduced by~\citet{zhihu2025entropy}, and was organized and extended by~\citet{cui2025entropy}. Proof can be seen in~\citet{zhihu2025entropy} and~\citet{cui2025entropy}. $\displaystyle{H}$ indicates the policy entropy of policy model, and $\text{Cov}$ denotes covariance, $\pi^k_\theta$ is the policy at step $k$, and $A^{\pi^k}(s,a)$ is the advantage function of action $a$ under state $s$. The covariance term becomes negative when high-advantage actions receive low probability under the current policy. In such cases, updates increase the policy entropy. Using an elevated temperature $t'>1$ amplifies this effect by further reducing the gap between high- and low-probability actions, making negative covariance more likely. As a result, UEC-RL induces controlled entropy increase specifically on difficult prompts, allowing the model to escape collapsed regimes and explore deeper reasoning trajectories that standard GRPO would fail to reach.

\begin{table*}[t]
 \caption{Comparison of Pass@1 accuracy on text and multimodal reasoning benchmarks. UEC-RL consistently outperforms the RL baselines. For AIME24 and AIME25, each question is repeated 32 times.}
 \vspace{2mm}
 \label{tab:main_results_all}
 \centering
 \resizebox{1\textwidth}{!}{%
  \renewcommand{\arraystretch}{1.2}
  \begin{tabular}{
    >{\raggedright\arraybackslash}m{5.0cm}  
    *{7}{>{\centering\arraybackslash}m{1.625cm}}| 
    >{\centering\arraybackslash}m{2.125cm}    
   }
   \toprule
   \textbf{Text Benchmarks} & \textbf{AIME24} & \textbf{AIME25} & \textbf{MATH}  & \textbf{GSM8K} & \textbf{Minerva} & \textbf{ARC} & \textbf{MMLU}  & \textbf{Average} \\
   \midrule
   \rowcolor{gray!10} Qwen2.5-math-7B  & 15.2 & 5.39 & 65.5 & 65.4 & 47.3 & 69.9 & 34.3 & 43.28 \\
   \quad+GRPO   & 25.8 & 9.27 & 77.6 & 87.1 & 29.0 & 78.6 & 45.0 & 50.34 \\
   \quad+DAPO   & 24.3 & 8.54 & 78.3 & 87.6 & 34.2 & 80.9 & 48.5 & 51.77 \\
   \quad+KL-cov   & 27.3 & 8.13 & 79.6 & \textbf{88.3} & 33.1 & 79.9 & 46.6 & 51.85 \\
   \quad+Entropy-Adv & 26.7 & 9.90 & 78.9 & 86.8 & 35.0 & 81.6 & 46.2  & 52.16 \\
   \rowcolor{blue!5}\quad+\textbf{UEC-RL} & \textbf{28.5} & \textbf{10.7} & \textbf{80.4} & 87.9 & \textbf{35.7} & \textbf{82.0} & \textbf{50.2} & \textbf{53.62} \\ 
   \rowcolor{blue!5} \quad\quad $\Delta$ vs. GRPO & \textbf{\textcolor{green!70!black}{+2.7}} & \textbf{\textcolor{green!70!black}{+1.43}} & \textbf{\textcolor{green!70!black}{+2.8}} &
   \textbf{\textcolor{green!70!black}{+0.8}} & 
   \textbf{\textcolor{green!70!black}{+6.7}} & 
   \textbf{\textcolor{green!70!black}{+3.4}} & 
   \textbf{\textcolor{green!70!black}{+2.3}} &
   \textbf{\textcolor{green!70!black}{+2.88}} \\
   \midrule
   \rowcolor{gray!10} Llama3.1-8B-Instruct  & 5.83 & 1.67 & 49.0 & 87.1 & 23.9 & 83.3 & 51.8 & 43.23 \\
   \quad+GRPO   & 8.02 & 1.67 & 54.6 & 87.6 & 26.8 & 84.3 & 53.4 & 45.20 \\
   \quad+DAPO   & 8.33 & 1.14 & 54.4 & 88.4 & 26.5 & 84.6 & 53.9 & 45.32 \\
   \quad+KL-cov & 8.22 & 0.83 & 54.0 & 88.2 & 27.6 & 84.6 & 54.5 & 45.42 \\
   \quad+Entropy-Adv & 8.02 & 1.67 & 54.0 & 87.4 & 27.2 & 84.3 & 52.3 & 44.98 \\
   \rowcolor{blue!5}\quad+\textbf{UEC-RL} & \textbf{9.27} & \textbf{2.08} & \textbf{56.6} & \textbf{88.8} & \textbf{28.3} & \textbf{85.3} & \textbf{55.9} & \textbf{46.61} \\ 
   \rowcolor{blue!5} \quad\quad $\Delta$ vs. GRPO & \textbf{\textcolor{green!70!black}{+1.25}} & \textbf{\textcolor{green!70!black}{+0.41}} & \textbf{\textcolor{green!70!black}{+2.0}} &
   \textbf{\textcolor{green!70!black}{+1.2}} & 
   \textbf{\textcolor{green!70!black}{+1.5}} &
   \textbf{\textcolor{green!70!black}{+1.0}} &
   \textbf{\textcolor{green!70!black}{+2.5}} &
   \textbf{\textcolor{green!70!black}{+1.41}} \\
   \midrule
   \end{tabular}
   }   
\resizebox{1\textwidth}{!}{%
  \renewcommand{\arraystretch}{1.2}
   \begin{tabular}{
    >{\raggedright\arraybackslash}m{5.0cm}  
    *{4}{>{\centering\arraybackslash}m{2.85cm}}| 
    >{\centering\arraybackslash}m{2.125cm}    
   }
   \textbf{Multimodal Benchmarks} & \textbf{MathVision} & \textbf{MathVerse} & \textbf{MathVista} & \textbf{We-Math} & \textbf{Average} \\
   \midrule
   \rowcolor{gray!10} Qwen2.5-VL-7B-Instruct  & 24.87 & 43.83 & 66.30 & 62.87 & 49.47 \\
   \quad+GRPO  & \textbf{29.11} & 47.51 & 72.60 & 67.53  & 54.19 \\
   \quad+DAPO  & 27.92 & 48.48 & 72.30 & 69.08  & 54.45 \\
   \quad+KL-cov  & 28.14 & 48.23 & 73.10 & 68.49  & 54.49 \\
   \quad+Entropy-Adv  & 27.86 & 48.63 & 71.80 & 68.62  & 54.23 \\
   \rowcolor{blue!5} \quad+\textbf{UEC-RL} & 
    28.82 & \textbf{49.34} & \textbf{73.40} & \textbf{69.48} & \textbf{55.26} \\
    \rowcolor{blue!5} \quad\quad $\Delta$ vs. GRPO & \textbf{\textcolor{red!70!black}{-0.29}} & \textbf{\textcolor{green!70!black}{+1.83}} & \textbf{\textcolor{green!70!black}{+0.80}} &
    \textbf{\textcolor{green!70!black}{+1.95}} & 
    \textbf{\textcolor{green!70!black}{+1.07}}\\
   \bottomrule
  \end{tabular}
 }
\end{table*}

\begin{figure*}[t]
 \centering
  \hspace*{0.15\textwidth}
  \includegraphics[height=0.17\textheight, keepaspectratio]{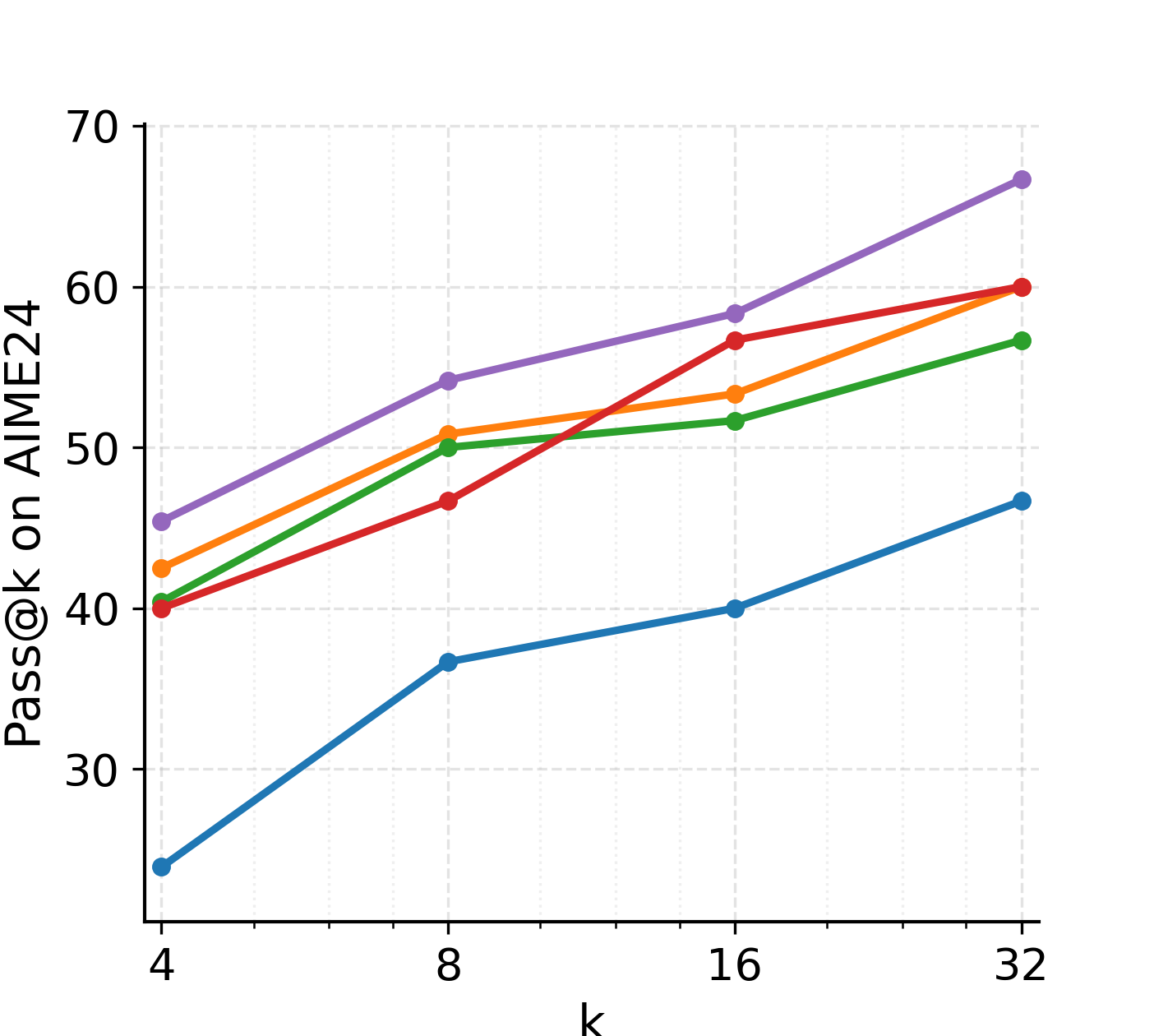}
  \includegraphics[height=0.17\textheight, keepaspectratio]{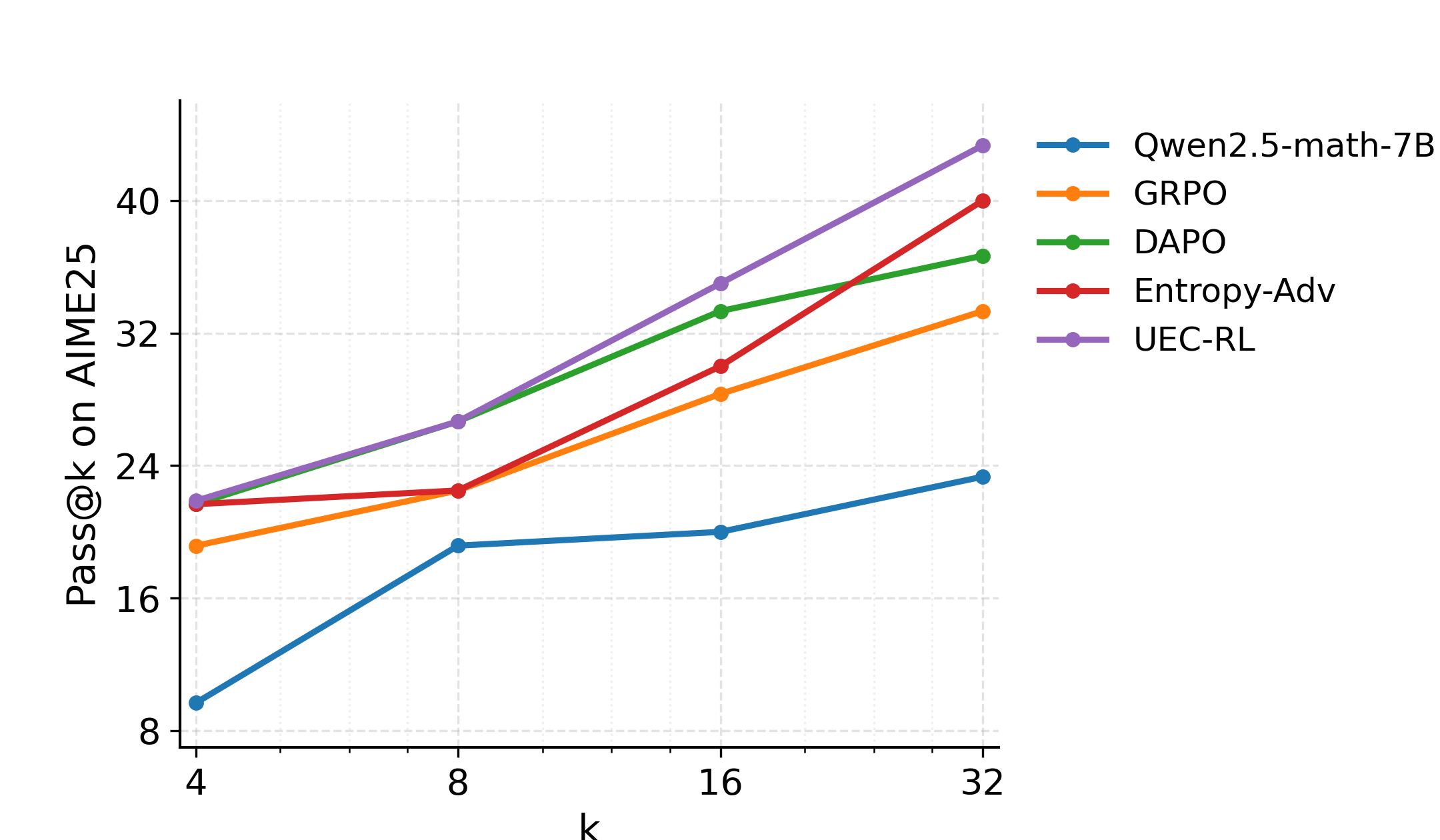}
 \caption{Pass@$k$ performance on the AIME24 and AIME25 benchmarks. UEC-RL consistently improves the success rate across different values of $k$.}   
 \label{fig:pass@k}
\end{figure*}

\begin{figure*}[t]
 \centering
  \begin{tabular}{ccccccc}
  \toprule
  Geometry3K & Qwen2.5-VL-7B-Instruct & UEC-RL & GRPO & DAPO & KL-cov & Entropy-Adv\\
  \midrule
  Accuracy & 38.44 & \textbf{55.41} & 50.75  & 49.09 & 47.09 & 50.91\\
  $\Delta$ vs. UEC-RL & - & - & 
  \textcolor{red!70!black}{-4.66} & 
  \textcolor{red!70!black}{-6.32} & 
  \textcolor{red!70!black}{-8.32} & 
  \textcolor{red!70!black}{-4.50}\\
  Time per step & - & 0.79$\times$ & 1$\times$  & 0.64$\times$ & 1$\times$ & 1$\times$ \\
  \bottomrule
 \end{tabular}\\
        \vspace{2mm}
  \includegraphics[width=0.3\linewidth]{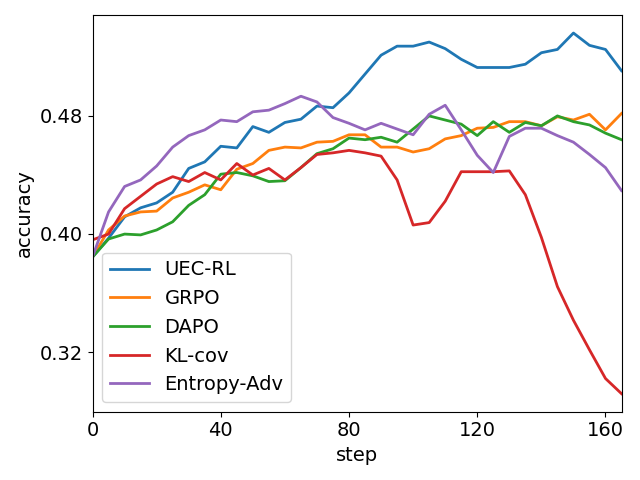}
  \includegraphics[width=0.3\linewidth]{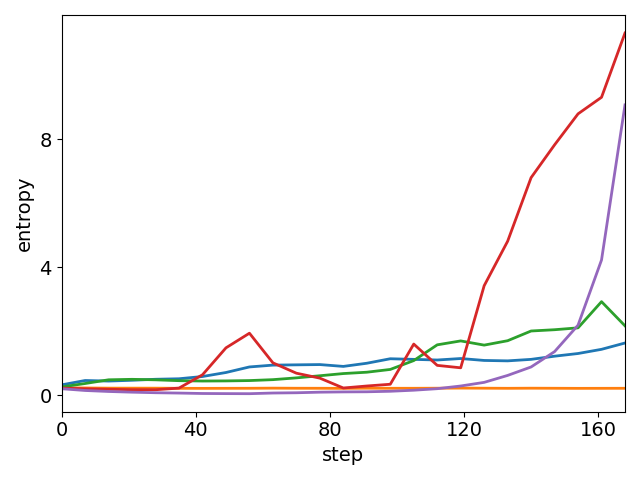}
  \includegraphics[width=0.3\linewidth]{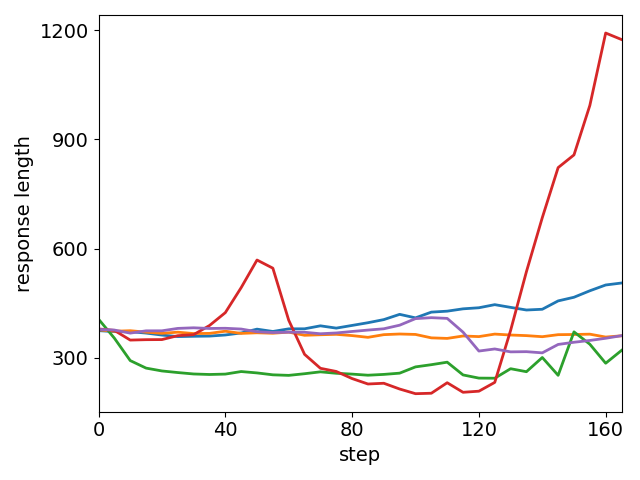}
    \caption{Results on Geometry3K. UEC-RL achieves the best accuracy, exhibits more stable entropy and response-length dynamics, and also demonstrates excellent training efficiency with lower per-step cost than GRPO.}
        
 \label{fig:GEO3K}
\end{figure*}

\subsection{Controllable Entropy Stabilizer}

Targeted exploration allows the policy to increase entropy on difficult prompts, but a complementary mechanism is required to prevent uncontrolled entropy growth and ensure convergence. UEC-RL introduces a controllable entropy stabilizer that repeatedly reinforces high-quality trajectories discovered during exploration.

Positive-advantage trajectories found under expanded sampling often have low initial probability and thus limited influence when used only once. Revisiting them strengthens their gradients and shifts probability mass toward correct reasoning patterns, producing a stabilizing effect on entropy.

\begin{theorem}[Entropy Stabilization]
\label{theo:replay_entropy}
Let $(q,o)$ be a trajectory with $A(q,o)>0$.  
If one update increases its log-likelihood,
\[
\log \pi_\theta^{k}(o\mid q) > \log \pi_\theta^{k-1}(o\mid q),
\]
then repeating this update (e.g., via replay) yields
\[
H(\pi_\theta^{k+1}) - H(\pi_\theta^{k}) < 0 .
\]
\end{theorem}

\begin{proof}
Because $A(q,o)>0$, raising $\pi_\theta(o\mid q)$ aligns high-advantage actions with high probability, producing a positive covariance term in Theorem~\ref{theo:entropy}. A positive covariance leads to decreasing entropy.
\end{proof}

Thus, exploration enlarges entropy only when needed, and the stabilizer gradually decreases entropy by consolidating informative trajectories. This interplay transitions training from exploration to stable convergence, avoiding both entropy collapse and divergence. Concretely, we form a candidate set from both regular and exploratory samples ($O_R \cup O_H$), filter trajectories whose advantages exceed a threshold $A_0$, and keep only the most recent $s'$ trajectories to prioritize up-to-date behaviors:
\begin{equation*}
\mathcal{O}_{S}
=\operatorname{Recent}_{s'}\!\left(\{\,o_i \in O_R \cup O_H : \hat{A}_{i,t}>A_0\,\}\right),
\end{equation*}
where $\operatorname{Recent}_{s'}(\cdot)$ returns the $s'$ most recent trajectories in generation order.

\section{Experiments}
We evaluate UEC-RL on both text-based and multimodal mathematical reasoning tasks. Our implementation is built upon EasyR1 and VeRL ~\cite{zheng2025easyr1,sheng2025hybridflow}. 
\paragraph{Datasets and benchmarks.}
We train UEC-RL on three datasets spanning both text-only and multimodal reasoning, and evaluate it on a comprehensive suite of text and multimodal benchmarks, with Geometry3K additionally used for in-domain analysis of training dynamics and ablations. Full details of the training datasets and evaluation benchmarks are provided in Appendix~\ref{app:datasets_benchmarks}.

\paragraph{Baseline.} For comparison, we include four representative RL baselines. Additional implementation details are provided in Appendix~\ref{app:impl}.
\begin{itemize}[leftmargin=1em]
    \item GRPO~\cite{shao2024deepseekmath}, the most widely adopted RL baseline;
    \item DAPO~\cite{yu2025dapo}, which enhances exploration through the clip-higher mechanism;
    \item KL-cov~\cite{cui2025entropy}, a covariance-aware KL regularization for entropy baseline;
    \item Entropy-Adv~\cite{cheng2025reasoning}, which encourages exploration by augmenting the advantage with an entropy term.
\end{itemize}

\begin{figure*}[t]
 \centering

  \includegraphics[width=0.3\linewidth]{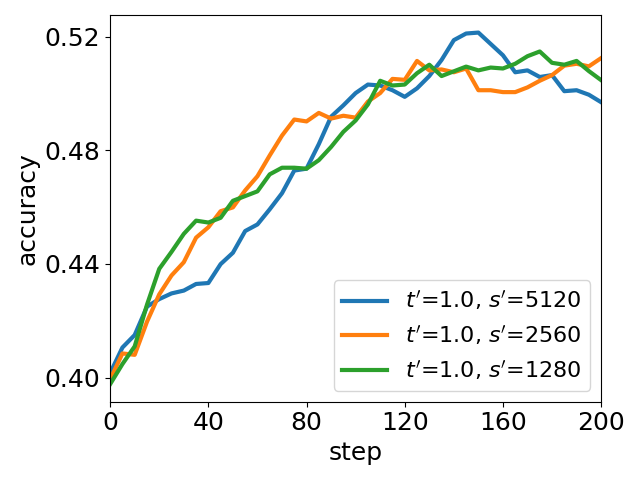}
  \includegraphics[width=0.3\linewidth]{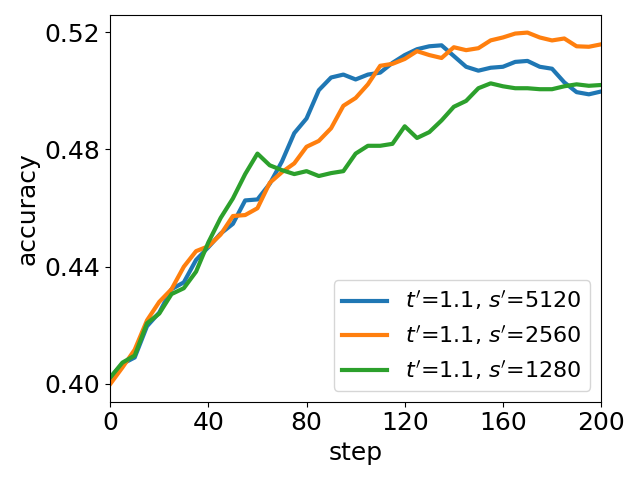}
  \includegraphics[width=0.3\linewidth]{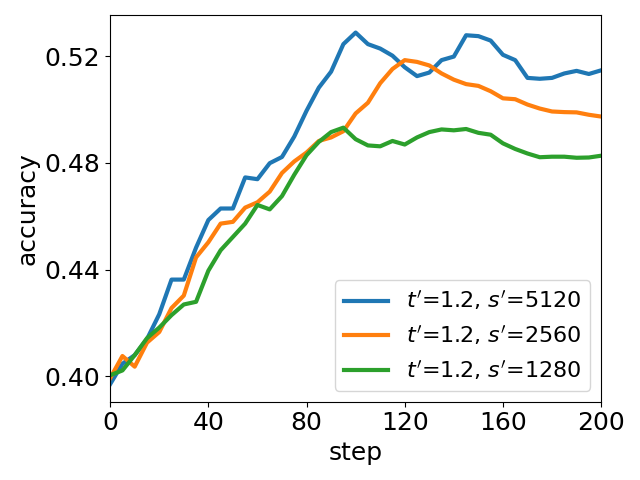}
  \\
  \includegraphics[width=0.3\linewidth]{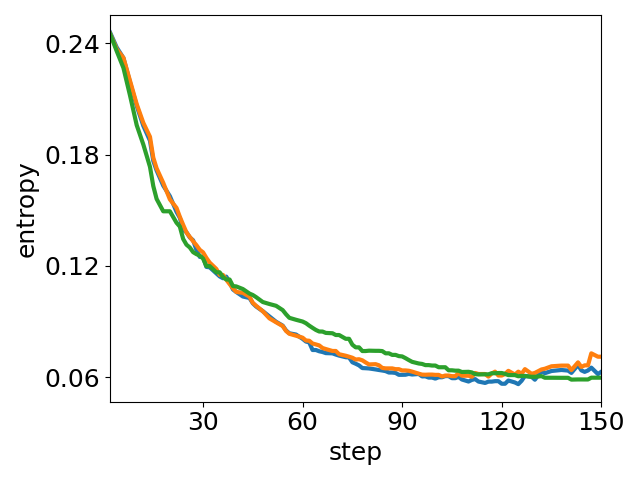}
  \includegraphics[width=0.3\linewidth]{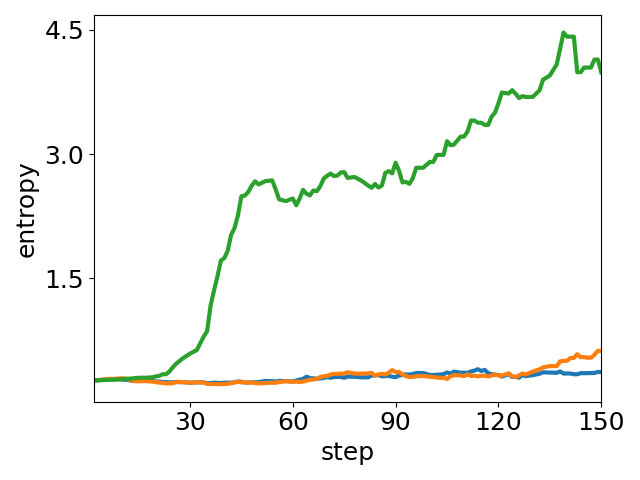}
  \includegraphics[width=0.3\linewidth]{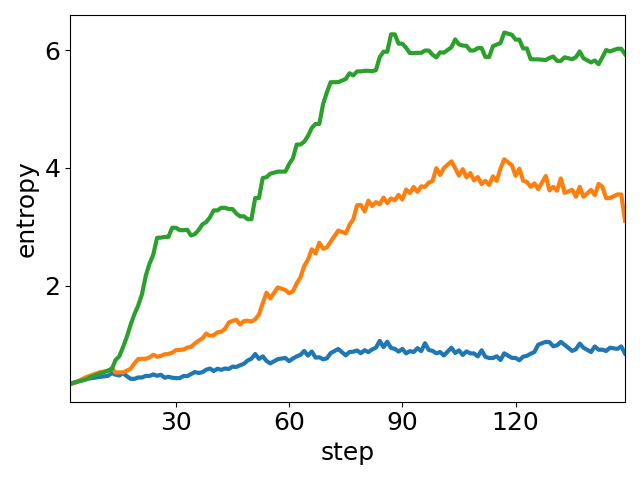}
  \\
    \vspace{0.5em}
    \begin{minipage}{0.98\textwidth}
        \centering
        \begin{tabular}{lccc}
        \toprule
        Accuracy (Entropy) & $t'=1.0$ & $t'=1.1$ & $t'=1.2$ \\
        \midrule
        $s'=5120$ & 53.74 (0.09) & 52.41 (0.31) & \underline{\textbf{55.41}} (0.73) \\
        $s'=2560$ & 52.75 (0.09) & \underline{54.41} (0.31) & 52.75 (2.37) \\
        $s'=1280$ & 52.41 (0.09) & 51.08 (2.44) & 50.25 (4.29) \\
        \bottomrule
        \end{tabular}
    \end{minipage}
 \caption{Ablation of parameter tuning: peak accuracy and average entropy under varying $t'$ and $s'$. Higher $t'$ boosts exploration (entropy↑), larger $s'$ aids stabilization (entropy↓), and best performance arises from the balance state of entropy.}
 \label{fig:paratuning}
\end{figure*}

\begin{figure*}[t]
 \centering
\begin{tabular}{ccccc}
  \toprule
  \quad Geometry3K  & UEC-RL  & w/o exploration  & w/o stabilizer \\
  \midrule
  Accuracy  & \textbf{55.41} & 50.41 \textcolor{red!70!black}{(-5.00)}  & 49.58 \textcolor{red!70!black}{(-5.93)}\\
  \bottomrule
 \end{tabular}\\
 \vspace{2mm}
  \includegraphics[width=0.3\linewidth]{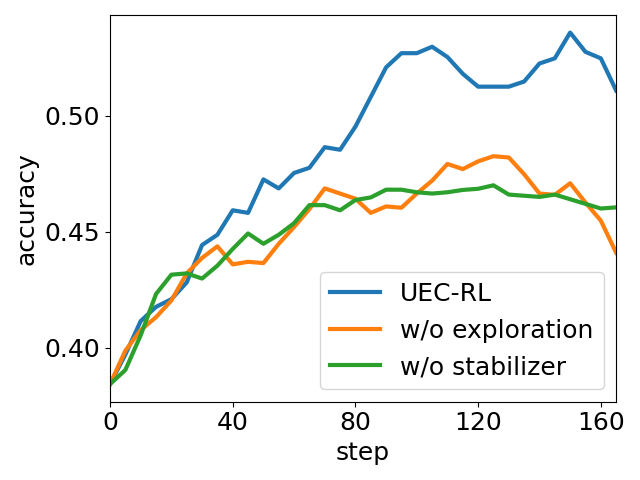}
  \includegraphics[width=0.3\linewidth]{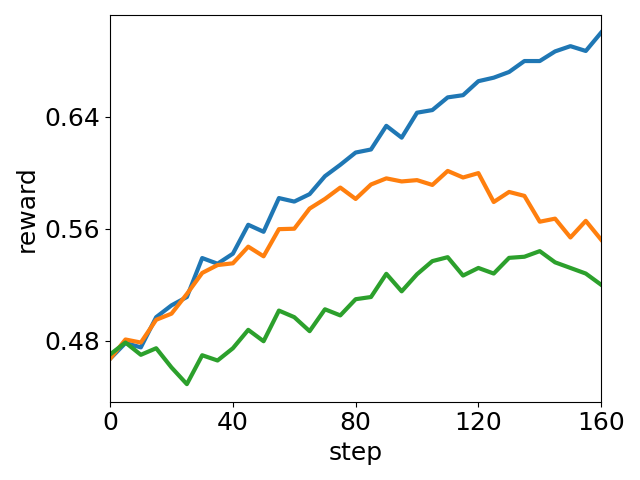}
  \includegraphics[width=0.3\linewidth]{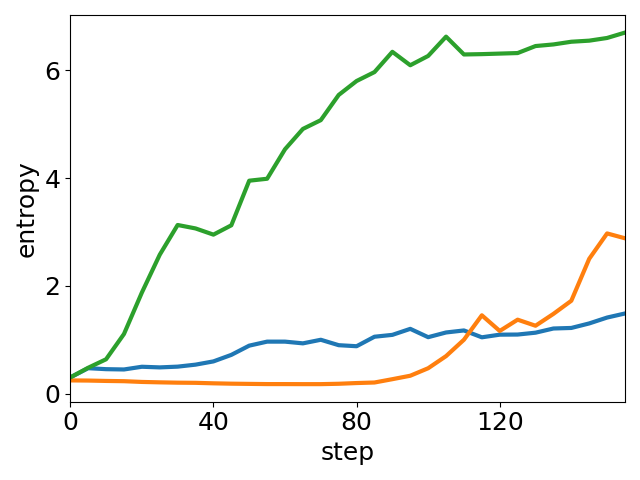}
 \caption{Module ablation of UEC-RL. Removing either exploration or stabilizer consistently degrades performance, highlighting their complementary roles in improving accuracy.}
 \label{ab}
\end{figure*}

\section{Main Results}

Table~\ref{tab:main_results_all} summarizes Pass@1 performance on text and multimodal reasoning benchmarks. Overall, UEC-RL consistently achieves the strongest or near-strongest results across different model families and task settings, showing that its bidirectional entropy-control mechanism generalizes beyond a specific backbone or modality.

\paragraph{Text reasoning.}
On Qwen2.5-Math-7B, UEC-RL achieves the best overall average among all RL baselines, outperforming GRPO, DAPO, KL-cov, and Entropy-Adv. The gains are particularly clear on AIME24, AIME25, MATH, and Minerva, showing that targeted exploration is especially beneficial for challenging mathematical reasoning tasks. These improvements further extend to Pass@$k$ evaluation (Figure~\ref{fig:pass@k}), where UEC-RL yields consistently stronger curves, indicating that the learned policy produces not only higher-quality but also more diverse reasoning trajectories. The advantage of UEC-RL also transfers to Llama3.1-8B-Instruct. UEC-RL achieves the best average score among all compared RL methods and improves over GRPO on most benchmarks. UEC-RL provides a robust optimization advantage across model families, suggesting that unified entropy control captures a more general property of RL-based reasoning training. Moreover, the gains extend beyond mathematical reasoning to the commonsense question answering benchmarks (ARC$_{\text{challenge}}$ and MMLU$_{\text{pro}}$). This cross-domain improvement suggests that the benefit of unified entropy control is not limited to math-specific reasoning, but generalizes to broader knowledge-intensive and commonsense reasoning tasks.

\paragraph{Multimodal reasoning.}
On multimodal benchmarks, UEC-RL again achieves the best overall average, outperforming the RL baselines across visually grounded reasoning tasks. These results show that UEC-RL remains effective even in multimodal settings, where training is typically less stable and gradient variance is larger. The consistent gains confirm that UEC-RL generalizes reliably from text-only reasoning to vision-language reasoning.

\paragraph{In-domain analysis on Geometry3K.}
We additionally conduct an in-domain analysis on Geometry3K, a challenging visually grounded mathematical benchmark. As shown in Figure~\ref{fig:GEO3K}, UEC-RL achieves a substantial accuracy gain over GRPO (\textbf{55.41} vs. 50.75), while also exhibiting more stable entropy dynamics and a smoother training trajectory. Existing exploration strategies designed for LLMs often fail to transfer to VLMs because visual reasoning amplifies gradient variance and makes entropy harder to control. In contrast, UEC-RL maintains effective exploration without destabilizing training, as reflected in both the entropy and response-length curves.
﻿
\paragraph{Efficiency.}
Figure~\ref{fig:GEO3K} also reports per-step training efficiency. UEC-RL runs at 0.79$\times$ the step time of GRPO, making it noticeably more efficient than GRPO and Entropy-Adv, while remaining slightly slower than DAPO. Importantly, this moderate computational cost delivers the strongest accuracy among all compared methods. UEC-RL improves reasoning performance without incurring the full cost typically associated with heavier exploration strategies.

\section{Ablation Study}

We conduct ablation experiments to study how the two key components of UEC-RL, namely the targeted exploration mechanism and the entropy-reducing stabilizer, contribute to training effectiveness. Our analysis consists of two parts: parameter tuning, which examines how exploration strength and stabilization capacity influence entropy dynamics, and module-level ablations, which isolate the effect of each component.

\subsection{Parameter tuning.}
To characterize how exploration and stabilization jointly determine model behavior, we vary the exploration temperature $t'$ and the stabilizer budget $s'$. As shown in Figure~\ref{fig:paratuning}, increasing $t'$ enlarges the exploration space and raises policy entropy, helping the model discover deeper reasoning chains on difficult prompts. This trend is consistent with Theorem~\ref{theo:entropy}, which predicts entropy increase when exploration encourages low-probability actions.

When $t'$ becomes large and $s'$ is insufficient, entropy grows rapidly and accuracy degrades. In contrast, increasing $s'$ strengthens the influence of high-quality trajectories and gradually stabilizes entropy, matching the entropy reduction effect described in Theorem~\ref{theo:replay_entropy}. The best performance appears when exploration and stabilization operate in balance. These results confirm that exploration enables the model to escape shallow local optima, while stabilization consolidates correct behaviors and maintains entropy within a desirable range.

\subsection{Module ablations.}
We further disable each component to isolate its contribution. Removing the exploration module suppresses entropy growth on difficult prompts and restores entropy collapse, leading to a reduction of $5.00$ accuracy points. Removing the stabilizer allows entropy to grow excessively and causes unstable reward dynamics, resulting in a reduction of $5.93$ accuracy points. Figure~\ref{ab} illustrates that exploration provides the necessary entropy increase to activate deeper reasoning, whereas the stabilizer prevents uncontrolled entropy drift and ensures convergence. Only the full UEC-RL framework exhibits both steady entropy regulation and consistent performance gains.

\section{Conclusion}
In this work, we presented UEC-RL, which addresses entropy collapse and provides a novel mechanism for bidirectional entropy regulation in RL for large language models. Empirically, UEC-RL delivers consistent improvements over strong RL baselines across both text-only and multimodal reasoning benchmarks. It improves Pass@1 performance while also strengthening Pass@$k$ under sampling, indicating not only higher single-trajectory accuracy but also a more diverse and reliable policy distribution. These results support the central claim that effective reasoning improvements require jointly promoting exploration and stabilizing optimization, rather than relying on one-sided clip-higher or entropy bonus.

Looking forward, UEC-RL can be extended in several directions. More accurate difficulty estimation could better identify when entropy should be increased, improving exploration efficiency. Adaptive scheduling of the exploration temperature and stabilizer budget could reduce task-specific tuning while maintaining the desired entropy regime. Finally, integrating UEC-RL with multi-step verification or agent-based reasoning systems may allow entropy control at finer granularity, guiding step-level branching and consolidation to further improve reliability and scalability in complex reasoning tasks.

\section*{Ethical considerations}
AI is only used for translation and language polishing in this paper.

\section*{Limitations}

While UEC-RL enables stable and controllable exploration, its effectiveness depends on selecting appropriate values for the exploration temperature $t'$ and stabilizer budget $s'$. Together, these hyperparameters determine the entropy range maintained during training. However, the entropy level that yields optimal performance is typically moderate, for example, around 0.5, and achieving it often requires task-specific hyperparameter configurations.

This variability arises because tasks differ substantially in difficulty: harder datasets require stronger exploration to escape local optima, whereas easier or lower-variance tasks benefit from tighter stabilization. Consequently, the $(t',s')$ combination that preserves a desirable entropy regime is not universal but instead depends on the intrinsic difficulty and variance structure of the training set. This makes UEC-RL relatively sensitive to hyperparameter choices, and achieving consistent performance across domains may require task-specific tuning.

Developing adaptive or self-regulating strategies that automatically calibrate $(t',s')$ based on task difficulty remains an important direction for future research.

\section*{Acknowledgments}
This work is supported by the Zhongguancun Academy (Grant No.s C20250203) and Natural Science Foundation of Tianjin(No.24JCQNJC02170). Weiran Huang is supported by National Natural Science Foundation of China (No. 62406192), Shanghai Municipal Special Program for Basic Research on General AI Foundation Models (Grant No. 2025SHZDZX025G03).

\bibliography{reference}
\appendix

\section{Datasets and benchmarks}
\label{app:datasets_benchmarks}
\paragraph{Training datasets.}
We train UEC-RL on three datasets that cover different modalities and difficulty levels:
\begin{itemize}[leftmargin=1em]
    \item \textbf{DAPO-17K}: a large-scale out-of-domain mathematical reasoning dataset designed to evaluate RL-based alignment algorithms for LLMs \cite{yu2025dapo}.
    \item \textbf{Multimodal dataset (6k)}: sampled from the multimodal corpora introduced in~\cite{wei2025advancing,wei2025unsupervised}, spanning a wide range of diagram, geometry, chart, and table problems. The dataset aggregates established resources including Geometry3K~\cite{lu2021inter}, GeoQA~\cite{chen2021geoqa}, GeoQA-Plus~\cite{cao2022augmented}, Geos~\cite{seo2015solving}, AI2D~\cite{kembhavi2016diagram}, TQA~\cite{kim2018textbook}, FigureQA~\cite{kahou2017figureqa}, TabMWP~\cite{lu2022dynamic}, ChartQA~\cite{masry2022chartqa}, IconQA~\cite{lu2021iconqa}, Clevr-Math~\cite{lindstrom2022clevr}, M3CoT~\cite{chen2024m}, and ScienceQA~\cite{lu2022learn}.
    \item \textbf{Geometry3K}: an in-domain geometric reasoning dataset used for detailed evaluation \cite{lu2021inter}.
\end{itemize}

\paragraph{Evaluation benchmarks.}
We assess UEC-RL across three categories of benchmarks:
\begin{itemize}[leftmargin=1em]
    \item \textbf{Text reasoning benchmarks.} We evaluate Pass@1 on seven widely used reasoning benchmarks: AIME24~\cite{hf_aime2024}, AIME25~\cite{hf_aime2024}, MATH~\cite{lightman2023lets}, GSM8K~\cite{cobbe2021gsm8k}, Minerva~\cite{lewkowycz2022solving}, ARC$_{\text{challenge}}$~\cite{clark2018think}, and MMLU$_{\text{pro}}$~\cite{wang2024mmlu}. These benchmarks span competition-level problems (AIME), formal mathematics (MATH), school-level word problems (GSM8K), scientific reasoning (Minerva), commonsense and science question answering (ARC), and broad knowledge-intensive multiple-choice reasoning across diverse subjects (MMLU), providing a comprehensive assessment of textual reasoning ability and generalization.
    \item \textbf{Multimodal reasoning benchmarks.}  We further evaluate on four challenging multimodal benchmarks:  MathVision~\cite{wang2024measuring}, MathVerse~\cite{zhang2024mathverse}, MathVista~\cite{lu2023mathvista}, and We-Math~\cite{qiao2024we}.  These benchmarks cover diverse visual formats—including diagrams, charts, tables, and multi-image compositions—and require integrating visual and symbolic reasoning.
    \item \textbf{Geometry3K in-domain dataset.}  To better understand the behavior of entropy-controlled RL, we conduct an in-depth analysis on Geometry3K \cite{lu2021inter}, including accuracy curves, entropy dynamics, response length behavior, and ablation studies.
\end{itemize}

\section{Implementation details \label{app:impl}}  
We follow the default EasyR1 configuration unless otherwise noted. Table~\ref{tab:impl_details} summarizes the hyperparameters for GRPO, DAPO, Entropy-Adv, and UEC-RL. For UEC-RL, difficult prompts trigger expanded exploration with $G'=20$ and temperature $t'=1.2$. Trajectories with advantages greater than 1 are stored in a replay buffer of size 5120, and replay is performed every 5 optimization steps. For each experiment setting, we run a single training run, save checkpoints every 10 optimization steps, and report the maximum performance achieved on each benchmark over all saved checkpoints.

\begin{table*}
\centering
\caption{Summary of implementation and evaluation details for all compared methods.}
\label{tab:impl_details}
\renewcommand{\arraystretch}{1.25}
\begin{tabular}{lccccc}
\toprule
\textbf{Settings of Training} & \textbf{GRPO} & \textbf{DAPO} & \textbf{Entropy-Adv} & \textbf{KL-cov} & \textbf{UEC-RL} \\
\midrule
\multicolumn{6}{l}{\textbf{Training settings}} \\
\midrule
Hardware & \multicolumn{5}{c}{8×A800 GPUs (40GB)} \\
Policy model init & \multicolumn{5}{c}{Qwen2.5-VL-7B-Instruct and Qwen2.5-Math-7B} \\
Max response length & \multicolumn{5}{c}{8192} \\
Batch size & \multicolumn{5}{c}{512} \\
Primary rollout $G$ & \multicolumn{5}{c}{5} \\
Learning rate & \multicolumn{5}{c}{$1\times10^{-6}$} \\
Temperature (training) & \multicolumn{5}{c}{1.0} \\
\midrule
$\epsilon_{\text{low}}$ & 0.2 & 0.2 & 0.2 & 0.2 & 0.2 \\
$\epsilon_{\text{high}}$ & 0.2 & 0.3 & 0.2 & 0.2 & 0.2 \\
Entropy bonus & -- & -- & $\beta,\kappa=0.4,2$ & see \citet{cui2025entropy} & -- \\
\midrule
Additional rollout $G'$ & -- & -- & -- & -- & 20 \\
Exploration temperature $t'$ & -- & -- & -- & -- & 1.2 \\
Replay buffer size $s'$ & -- & -- & -- & -- & 5120 \\
Replay frequency & -- & -- & -- & -- & 5 steps \\
Replay criterion & -- & -- & -- & -- & $\hat{A} > 1$ \\
\midrule
\multicolumn{6}{l}{\textbf{Settings of evaluation }} \\
\midrule
Max response length (eval) & \multicolumn{5}{c}{8192} \\
Temperature (eval) & \multicolumn{5}{c}{0.2} \\
Top-$p$ (eval) & \multicolumn{5}{c}{0.95} \\
\bottomrule
\end{tabular}
\end{table*}

\end{document}